\documentclass[letterpaper, 10 pt, conference]{ieeeconf}  
\usepackage{amsmath}
\usepackage{bm}
\usepackage{amssymb}
\usepackage{xcolor}
\usepackage{algorithm,algorithmic}
\usepackage{booktabs}
\usepackage{graphicx}
\usepackage{cite}

\newtheorem{theorem}{Theorem}
\newtheorem{lemma}{Lemma}

\newtheorem{remark}{Remark}

\newcommand{\R}{\mathbb{R}}

\renewcommand{\S}{\mathbb{S}}

\newcommand{\diag}{\operatorname{diag}}

\IEEEoverridecommandlockouts                              

\title{\LARGE \bf
LMI-Net: Linear Matrix Inequality--Constrained Neural Networks via Differentiable Projection Layers
}

\author{Sunbochen Tang, Andrea Goertzen, and Navid Azizan
\thanks{The authors are with the Laboratory for Information and Decision Systems (LIDS), Massachusetts Institute of Technology, Cambridge, MA 02139, USA. Emails: {\tt\small \{tangsun, agoertz, azizan\}@mit.edu}}%
}

\begin{document}

\maketitle
\thispagestyle{empty}
\pagestyle{empty}

\begin{abstract}
    Linear matrix inequalities (LMIs) have played a central role in certifying stability, robustness, and forward invariance of dynamical systems. Despite rapid development in learning-based methods for control design and certificate synthesis, existing approaches often fail to preserve the hard matrix inequality constraints required for formal guarantees. We propose LMI-Net, an efficient and modular differentiable projection layer that enforces LMI constraints by construction. 
    Our approach lifts the set defined by LMI constraints into the intersection of an affine equality constraint and the positive semidefinite cone, performs the forward pass via Douglas–Rachford splitting, and supports efficient backward propagation through implicit differentiation. We establish theoretical guarantees that the projection layer converges to a feasible point, certifying that LMI-Net transforms a generic neural network into a reliable model satisfying LMI constraints. Evaluated on experiments including invariant ellipsoid synthesis and joint controller-and-certificate design for a family of disturbed linear systems, LMI-Net substantially improves feasibility over soft-constrained models under distribution shift while retaining fast inference speed, bridging semidefinite-program-based certification and modern learning techniques.
    
\end{abstract}

\section{Introduction}

Linear matrix inequalities (LMIs) have served as a unifying framework across a wide variety of important problems in dynamics and control \cite{boyd1994linear}, including stability certificate synthesis, robust invariance analysis, and control design. Although a single LMI-constrained problem can often be handled efficiently offline by numerical solvers, many applications involve families of related instances in which the same semidefinite programming problem template must be solved repeatedly \cite{alessio2009survey, apkarian2000parameterized, sambharya2023end, grontas2026pinet}. For example, this repeated parameterized structure appears when system parameters or operating conditions change, or disturbance descriptions vary. 
Under such settings, it is desirable to shift as much computation as possible offline by computing a function that maps problem parameters to solutions, so that a new instance can be handled by function evaluation rather than by solving a new optimization problem from scratch. This offline-online decomposition is closely related to the philosophy of explicit MPC \cite{alessio2009survey}, which yields piecewise-affine explicit control laws associated with multiparametric quadratic programming formulations. Since such piecewise-affine maps do not exist in general \cite{bellon2025parametric}, a learned optimizer that satisfies LMI constraints can be highly desirable for enabling fast online evaluation.


Recent years have seen rapid progress in learning-based methods for certificate synthesis and control design \cite{dawson2023safe}, including approaches that learn stability and safety certificates from data \cite{kolter2019learning, boffi2021learning, tang2024learning, goertzen2025eco}, as well as methods that jointly learn certificates and feedback control policies \cite{lindemann2021learning, min2023data, rezazadeh2022learning, tsukamoto2021contraction}. Because the objective is to establish certifiable stability or safety, satisfaction of the underlying certificate and controller constraints is essential \cite{dawson2023safe}. However, obtaining provable guarantees on the behavior of expressive neural models beyond the training data remains difficult, and constraint violations on previously unseen instances can invalidate the resulting certification \cite{boffi2021learning}.

A common way to encourage constraint satisfaction is to augment the training objective with sample-based regularization terms that penalize constraint violations \cite{dawson2023safe, tsukamoto2021contraction}. While such soft-constrained formulations can be effective empirically, they do not in general guarantee constraint satisfaction at inference time \cite{min2024hardnet, tordesillas2023rayen}, especially on inputs outside the training distribution. A complementary line of work therefore seeks hard feasibility by construction by designing differentiable layers that enforce constraints on the neural network output, as in \cite{kolter2019learning, min2024hardnet, tordesillas2023rayen}. These methods design enforcement mechanisms based on the structure of the constraints. Affine constraints are addressed in \cite{kolter2019learning, min2024hardnet} by designing closed-form projection layers, and convex constraints can be enforced by ray-based feasible parameterization in \cite{tordesillas2023rayen}, although the method is restricted to input-independent constraints. 
This leaves open the design of an efficient differentiable projection layer for LMI-constrained learning problems.

We address this challenge with an explicit, differentiable projection layer tailored to LMI constraints. The key observation driving our approach is that the feasible set of a parameterized LMI admits a lifted representation as the intersection of an affine equality constraint and the positive semidefinite cone. Leveraging this structure, we develop a projection mechanism based on the Douglas–Rachford algorithm~\cite{lions1979splitting}, an iterative splitting method recently shown to be effective in constrained learning contexts~\cite{sambharya2023end,grontas2026pinet}. Our approach enables both efficient forward-pass projections onto the decomposed constraint sets and implicit differentiation in the backward pass. The resulting layer is backbone-agnostic and converts repeated constrained optimization into a learned computation, enabling fast evaluation while satisfying LMI constraints at inference time. Our framework is, to our knowledge, the first to enforce input-dependent LMI constraints directly on neural network outputs, offering a scalable pathway for integrating convex control constraints into learning systems.


Our operator-splitting approach enables principled enforcement of LMI constraints within data-driven models. Our key contributions are as follows:
\begin{itemize}
    \item We develop a tailored splitting scheme for LMI constraints that decomposes the feasible set into components with tractable structure. This formulation admits explicit and efficient projections onto each set, making it well-suited for integration into differentiable architectures.
    \item We establish theoretical convergence guarantees by leveraging classical results for the Douglas–Rachford algorithm. These guarantees provide a rigorous foundation for the correctness and stability of the proposed projection layer.
    \item We present experimental results demonstrating consistent constraint satisfaction and stable behavior across a range of settings. In particular, our approach maintains reliability under both in-distribution and out-of-distribution inputs, highlighting its robustness in practical deployment scenarios.
\end{itemize}

\section{Problem Formulation}

\subsection{Parameterized Optimization under LMI Constraints}
Consider the parameterized optimization problem,
\begin{equation}\label{eq:LMI-opt}
    \begin{split}
        \min_{y} c(y), \text{subject to } F(y; \xi) \triangleq F_0(\xi) + \sum_{i=1}^m y_i F_i(\xi) \succeq 0,
    \end{split}
\end{equation}
where the symmetric matrices $F_i(\xi) \in \S^n, i \in \{0, ..., m\}$ are known and parameterized by $\xi \in \R^p$, $c(\cdot)$ is a cost function. $y \in \R^m$ is the decision variable, and the optimal solution is a function of $\xi$, $y^*(\xi)$. Note that $F(y; \xi) \succeq 0$ is a general form that can represent a set of LMIs, as multiple LMIs $F^{(1)}(y; \xi) \succeq 0, ..., F^{(N)}(y; \xi) \succeq 0$ can be reformulated as $F(y; \xi) = \diag{(F^{(1)}(y; \xi), ..., F^{(N)}(y; \xi))}\succeq 0$.

The optimization in \eqref{eq:LMI-opt} is a reduction of many problems in control theory, where the structure of the objective function $c(\cdot)$ and constraints remain fixed for a family of systems, and $\xi$ encodes system-specific parameters. For example, synthesizing a Lyapunov stability certificate for a family of linear systems, $\{\dot{x} = Ax: A \in S(A)\}$ where $S(A)$ is a set of Hurwitz matrices, can be formulated as solving a semidefinite program parameterized by $A$. Given a specific $A$, the stability certificate synthesis problem is then an instance of the parameterized optimization problem. Instead of repeatedly invoking a numerical solver for each different instance of $A$, learning a map $\xi\to y^*(\xi)$ that approximates the optimizer can significantly speed up computation, which is especially beneficial in real-time or distributed settings. 

\subsection{Self-supervised Learning with Feasibility by Construction}
The objective is to learn a neural network $\hat{y}(\xi; \theta)$ that approximates the optimal solution of \eqref{eq:LMI-opt}. Instead of using labeled data that relies on a solver to provide supervised signals, we adopt a self-supervised setting, where we define the training loss to be the average cost function value over different $\xi$. More formally, given a dataset consisting of $\xi$ drawn from a known distribution, $D_{\text{train}} = \{\xi^{(i)}\}_{i=1}^{N_{\text{train}}}$, the training process solves the following optimization problem,
\begin{subequations}\label{eq:learning}
\begin{align}
    L(\theta) &= \frac{1}{N_{\text{train}}}\sum_{i=1}^{N_{\text{train}}} c(\hat{y}(\xi^{(i)};\theta)), \label{eq:learning-obj}\\
    \text{subject to} \quad &F(\hat{y}(\xi; \theta); \xi) \succeq 0,\quad \forall \xi \in \Xi, \label{eq:learning-constr}
\end{align}
\end{subequations}
where $\Xi \subset \R^p$ is the set of admissible parameter values.
Our goal is to design a learned optimizer that is \textit{feasible by construction}. The constraint in \eqref{eq:learning-constr} needs to be satisfied for all admissible parameters $\xi\in \Xi$, not just for those in the training data $D_{\text{train}}$. 

For problems such as stability certificate synthesis and control design, feasibility by construction is highly desirable, enabling formal guarantees of safety and stability in physical system applications. 
To enforce feasibility, we define a context-dependent projection, $\Pi_{\xi}: \R^m \to \R^m$ that maps any generic neural network output $y_{\text{NN}}(\xi; \theta) \in \R^m$ to a feasible point $\hat{y}(\xi; \theta) = \Pi_{\xi}(y_{\text{NN}}(\xi;\theta))$ satisfying $F(\hat{y}(\xi; \theta); \xi) \succeq 0$. Key requirements for such a projection operator $\Pi_{\xi}$ include: (i) the output needs to be provably feasible for all inputs; (ii) $\Pi_{\xi}$ needs to be fully differentiable for training purposes; (iii) $\Pi_{\xi}$ is computationally efficient during inference.
As studied extensively in the literature \cite{boyd1994linear}, developing an explicit projection operator for an LMI constraint is difficult in general. We address this issue by decomposing the LMI constraint into the intersection of an affine constraint and a positive semidefinite cone constraint, and leveraging the Douglas-Rachford algorithm for efficient computation and provable convergence to a feasible point. The details of our approach are discussed in Section~\ref{sec:DR}.

\subsection{Illustrative Example: Ellipsoidal Invariant Sets for Disturbed Linear Systems}\label{sec:example}
Consider a linear system under disturbance,
\begin{equation}\label{eq:exp1-sys}
    \dot{x} = Ax + B_w w, \quad w^T w \leq 1,
\end{equation}
where $A \in \R^{n \times n}$ is a Hurwitz matrix, $B_w \in \R^{n \times n_w}$, and $w$ is a norm-bounded disturbance. The goal is to find an ellipsoidal set $\mathcal{E}(P) = \{x : x^T P x \leq 1\}$ with $P \succ 0$ that is forward invariant for \eqref{eq:exp1-sys}.

Consider a Lyapunov-like storage function candidate $V(x) = x^T P x$ with $P \succ 0$. A sufficient condition for $\mathcal{E}(P)$ to be robustly invariant is
\begin{equation}\label{eq:exp1-suff}
    \dot{V}(x,w) \leq 0 \quad \forall (x,w) \text{ s.t. } w^T w \leq 1, \; V(x) \geq 1,
\end{equation}
where $\dot{V}(x,w) = x^T(A^T P + PA)x + 2x^T P B_w w$. Using the S-procedure \cite{yakubovich1977sprocedure}, \eqref{eq:exp1-suff} holds if there exist $\alpha, \beta \geq 0$ such that for all $x$ and $w$,
\begin{equation*}
    \dot{V}(x,w) - \beta (w^T w - 1) + \alpha(x^T P x - 1) \leq 0.
\end{equation*}
Without loss of generality, assuming $\alpha = \beta \geq 0$, the above condition, combined with $P \succ 0$, simplifies  to
\begin{equation}\label{eq:exp1-lmi-simple}
    \begin{bmatrix} A^T P + PA - \alpha P & PB_w \\ B_w^T P & -\alpha I \end{bmatrix} \preceq 0, \qquad P \succeq \epsilon I
\end{equation}
For fixed $\alpha \geq 0$ and $\epsilon > 0$, the optimization problem is: find $P$ subject to \eqref{eq:exp1-lmi-simple}. The objective can be chosen as minimizing the volume of $\mathcal{E}(P)$, i.e., $\min -\log \det P$.

Since $P \in \S^n$, it has $m = \frac{n(n+1)}{2}$ degrees of freedom. Let $\{E_j\}_{j=1}^{m}$ be an orthonormal basis for $\S^n$ under the Frobenius inner product. We parameterize $P$ as
\begin{equation}\label{eq:P-param}
    P(y) = \sum_{j=1}^{m} y_j E_j, \quad y \in \R^m.
\end{equation}
Substituting into \eqref{eq:exp1-lmi-simple}, the constraint takes the standard form \eqref{eq:LMI-opt} with $\xi = (A, B_w)$. The learned mapping is $\xi \mapsto y$, where the neural network outputs the coefficients of a feasible $P$.

\section{Differentiable Projection Layers via Douglas-Rachford Splitting}\label{sec:DR}
\subsection{The Douglas-Rachford Algorithm}
The Douglas-Rachford algorithm is used to solve optimization problems of the form

\begin{equation}\label{eq:DR-obj}
    \underset{z}{\text{argmin }}f_\xi(z)+g_\xi(z).
\end{equation}

Here, $z\in\mathbb{R}^l$ is a decision variable, $\xi\in\mathbb{R}^p$ is a context parameter, and $f_\xi:\mathbb{R}^l\to\mathbb{R}$ and $g_\xi:\mathbb{R}^l\to\mathbb{R}$ are convex objectives. 
Douglas-Rachford solves the combined objective via an iterative method that alternates between proximal and reflection steps. Specifically, a Douglas-Rachford iteration is performed as follows.
\begin{equation}\label{eq:DR-alg}
    \begin{alignedat}{2}
        {w}_{k+1}&=\text{prox}_{f\sigma}(z_k) &&\quad\quad\text{proximal point of }f\\
        \bar{w}_{k+1}&=2{w}_{k+1} -z_k&&\quad\quad\text{reflection across } w_{k+1}\\ 
        {v}_{k+1} &= \text{prox}_{g\sigma}(\bar{w}_{k+1})&&\quad\quad\text{proximal point of }g\\
        \bar{v}_{k+1}&=2{v}_{k+1} -\bar{w}_{k+1}&&\quad\quad\text{reflection across }v_{k+1}\\
        z_{k+1}&=\frac{\bar{v}_{k+1}+z_k}{2}&&\quad\quad\text{averaging with current iterate}
    \end{alignedat}
\end{equation}

The objectives are incorporated to each iteration via the proximal operator, $\text{prox}_{h\sigma}(\bar{z})=\underset{z}{\text{argmin }}h(z)+\frac{1}{2\sigma}||\bar{z}-z||^2$, which balances minimizing the specific objective $h$ with remaining close to the input point $\bar{z}$. Douglas-Rachford is useful for problems in which the combined objective in \eqref{eq:DR-obj} is difficult to solve but the proximal operator for both $f_\xi(z)$ and $g_\xi(z)$ is straightforward to compute, particularly when the solutions can be evaluated in closed-form. 
\subsection{Douglas-Rachford for Feasibility Problems}
The Douglas-Rachford algorithm can be used to solve feasibility problems by formulating constraint satisfaction as a convex optimization problem. Consider two convex constraint sets $\mathcal{C}_1$ and $\mathcal{C}_2$ with $\mathcal{C}_1\cap\mathcal{C}_2\neq\emptyset$. Define $f(z)$ and $g(z)$ in \eqref{eq:DR-obj} with $\mathcal{I}_{\mathcal{C}_1}(z)$ and $\mathcal{I}_{\mathcal{C}_2}(z)$, respectively. $\mathcal{I}_{\mathcal{C}_i}(z)$ is defined to be $0$ when the constraint $z\in\mathcal{C}_i$ is satisfied and $\infty$ otherwise. With this definition for $f(z)$ and $g(z)$, it is clear that the solution to \eqref{eq:DR-obj} occurs only when both constraints are satisfied. Computing the proximal solution for each of the two sets is therefore equivalent to solving $\text{prox}_{\mathcal{C}_i\sigma}(\bar{z}) = \underset{z}{\text{argmin }}\mathcal{I}_{\mathcal{C}_i}(z)+\frac{1}{2\sigma}||\bar{z}-z||^2$. That is, for the feasibility problem, the proximal solution step in \eqref{eq:DR-alg} is simply the Euclidean projection onto the respective constraint. Although we drop the dependence on the context parameter $\xi$ for $f_\xi$ and $g_\xi$ for notational convenience, we emphasize that Douglas-Rachford readily handles context-dependent constraints, so long as they remain convex.

For a constraint set $\mathcal{C}$, Douglas-Rachford offers an efficient projection of points onto the feasible set when $\mathcal{C}$ can be decomposed into two sets $\mathcal{C}_1$ and $\mathcal{C}_2$ (with $\mathcal{C}=\mathcal{C}_1\cap\mathcal{C}_2$) whose projections are readily computable. While many splits for $\mathcal{C}_1$ and $\mathcal{C}_2$ may exist, selecting a splitting scheme where the respective Euclidean projections onto each constraint set are computable in closed-form reduces computational burden. Therefore, the efficiency provided by Douglas-Rachford for feasibility problems is often ultimately enabled by the choice of splitting scheme, making the selection of the right scheme for the right problem an important strategic objective. 


\subsection{LMI as the Intersection of Two Convex Sets}


In this work, we propose a splitting scheme that decomposes the LMI condition $F(y)\succeq0$ into the intersection of an affine equality condition and the positive semidefinite cone. Although the projection onto $F(y)\succeq0$ is generally intractable, our splitting scheme enables efficient projection onto the two individual sets. 

\begin{equation}\label{eq:LMI_set}
\begin{split}
    \mathcal{C}=\Pi_m(\mathcal{C}_1\cap\mathcal{C}_2),\quad\mathcal{C}_1&=\left\{\begin{bmatrix}y\\x\end{bmatrix}:F(y)=X\right\},\\\quad\mathcal{C}_2&=\left\{\begin{bmatrix}y\\x\end{bmatrix}:y\in\mathbb{R}^m,X\succeq0\right\}.
    \end{split}
\end{equation}
Here, $x=\text{vec}(X)\in\mathbb{R}^{n^2}$ is an auxiliary variable included as an intermediate to enable efficient projection onto the positive semidefinite cone. 
The projection $\Pi_m:\mathbb{R}^{m+n^2}\to\mathbb{R}^m$ is a final projection onto the first $m$ entries of $z=[y^T \,\,\,x^T]^T$, selecting only $y$ as an output and ignoring the auxiliary variable $x$. 
The intersection $\mathcal{C}_1\cap\mathcal{C}_2$ is clearly equivalent to the LMI condition $F(y)\succeq0$. We define an optimization problem of the form \eqref{eq:DR-obj} for the closest projection of the neural network output $\hat{y}_\theta$ onto the constraint $\mathcal{C}$.
\begin{equation}\label{eq:enforce-obj}
    z^*=\underset{z}{\text{argmin }}||\hat{y}_\theta-y||^2 + \mathcal{I}_{\mathcal{C}_1}(z) +\mathcal{I}_{\mathcal{C}_2}(z)
\end{equation}

Note that $y=\Pi_m(z)=z_{1:m}$. Equation~\eqref{eq:enforce-obj} can be separated into two convex objective functions $f(z)=||y-\hat{y}_\theta||^2 + \mathcal{I}_{\mathcal{C}_1}(z)$ and $g(z)=\mathcal{I}_{\mathcal{C}_2}(z)$. We now propose efficient computations of the proximal operator for each of the two objectives. We use $\bar{\cdot}$ to denote variables that are the input to the proximal projection. For $g(z)=\mathcal{I}_{\mathcal{C}_2}(z)$, the proximal operator is the minimum-distance projection onto the set $\mathcal{C}_2$. We compute this projection efficiently via an eigenvalue clipping operation. Specifically, let $\bar{X}=U\Lambda U^T$ be the eigendecomposition of the symmetric matrix $\bar{X}$. We can then define the projection
\begin{equation}\label{eq:proj-C2}
    \Pi_{\mathcal{C}_2}(\bar{X}) = U\text{max}(0,\Lambda)U^T.
\end{equation}
The max operation is applied element-wise to the eigenvalue matrix $\Lambda$. Equation~\eqref{eq:proj-C2} clearly outputs a positive semidefinite matrix.

We now define a closed-form solution for the projection onto the constraint $\mathcal{C}_1$. For $f(z)=||\hat{y}_\theta-y||^2 + \mathcal{I}_{\mathcal{C}_1}(z)$, the proximal operator is $\text{prox}_{f\sigma}(\bar{z})=\underset{z}{\text{argmin }}||y-\hat{y}_\theta||^2 + \mathcal{I}_{\mathcal{C}_1}(z)+\frac{1}{2\sigma}||\bar{z}-z||^2$. This differs from a Euclidean projection from the input point $\bar{z}$, because it balances minimizing the distance between the projected point and both the input point $\bar{z}$ \textit{and} the model output $\hat{y}_\theta$. The tradeoff between these two objectives for a given Douglas-Rachford iteration is tuned with $\sigma$. By expanding the competing objectives and completing the square, we can write the objective as a Euclidean projection from a point that is a weighted average of $\hat{y}_\theta$ and $\bar{z}$. 
\begin{equation}\label{eq:C1-prox}
    \begin{split}
    \text{prox}_{f\sigma}(\bar{z})=\underset{z}{\text{argmin }} \frac{1}{2}\left|\left|y-\frac{1}{2\sigma+1}(2\sigma\hat{y}_\theta+\bar{y})\right|\right|^2 +\\  \frac{1}{2} ||X-\bar{X}||_F^2+ \mathcal{I}_{\mathcal{C}_1}(z)
    \end{split}
\end{equation}

Note that the optimization variable $z$ includes both $y$ and the auxiliary variable $X$. We now define a closed-form solution for the projection onto constraint $\mathcal{C}_1$. Our proposed constraint decomposition scheme in \eqref{eq:LMI_set} makes the necessary projection onto $\mathcal{C}_1$ linear in $[y^T\,\,\,x^T]^T$, enabling the use of a closed-form linear equality constraint projection. We define the projection 
\begin{equation}\label{eq:proj-C1-int1}
\begin{split}
    \Pi_{\mathcal{C}_1}\left(\begin{bmatrix}\bar{y}\\\bar{x}\end{bmatrix},\hat{y}_\theta\right) &= \underset{y,X}{\text{argmin }}\frac{1}{2}||y-y_\text{avg}||^2+\frac{1}{2}||X-\bar{X}||_F^2\\
    &\text{subject to } F(y)=X,
    \end{split}
\end{equation}
where $y_\text{avg} = \frac{1}{2\sigma+1}(2\sigma\hat{y}_\theta+\bar{y})$. Note that there is no need to define an $X_\text{avg}$, since $X$ is an auxiliary variable that does not have a corresponding neural network output. By vectorizing the matrix $X$, this problem becomes a Euclidean distance minimization subject to linear constraints on $y$ and $X$. We vectorize $\bar{x}=\text{vec}(\bar{X})$, $c=\text{vec}(F_0)$, and $L=\left[\text{vec}(F_1),\text{vec}(F_2),...,\text{vec}(F_m)\right]$. The constraint $F(y)=X$ is now defined by $Ly+c=x$, a linear combination of vectors, rather than a linear combination of matrices. We substitute the constraint into \eqref{eq:proj-C1-int1}, 
\begin{equation}\label{eq:proj-C1-int2}
    \Pi_{\mathcal{C}_1}\left(\begin{bmatrix}\bar{y}\\\bar{x}\end{bmatrix}\right) = \underset{y,x}{\text{argmin }}\frac{1}{2}||y-y_\text{avg}||^2+\frac{1}{2}||Ly+c-\bar{x}||^2,\\
\end{equation}
A closed-form solution to \eqref{eq:proj-C1-int2} can be computed by setting the gradient of the objective to $0$. This gives the closed-form projection.
\begin{equation}\label{eq:proj-C1}
    \begin{split} &\Pi_{\mathcal{C}_1}\left(\begin{bmatrix}\bar{y}\\\bar{x}\end{bmatrix},\hat{y}_\theta\right)=\begin{bmatrix}
        y^*\\Ly^*+c
    \end{bmatrix}\\&y^*=\left(I+L^T L\right)^{-1}\left(y_\text{avg}-L^T(c-\bar{x})\right)
    \end{split}
\end{equation}


The alternating projection proposed can be summarized as a decomposition of the projection onto $\mathcal{C}$ into two sub-projections that are efficiently computable in closed form. In practice, we exploit the symmetry of $X$ to solve for only $n(n+1)/2$ auxiliary variables, using a weighted matrix for the projection $\Pi_{\mathcal{C}_1}$ such that the projection is computed with respect to the Frobenius norm of the full matrix $X$ as in Equation~\eqref{eq:proj-C1-int2}. The projection layer here is backbone-agnostic, meaning it can be used with any neural network backbone to enforce LMI constraints. Algorithm~\ref{alg:LMI} describes the end-to-end constraint enforcement procedure using Douglas-Rachford. 
\begin{algorithm}
\caption{Neural Network with LMI Constraint Enforcement Forward Pass}
\begin{algorithmic}[1]
\label{alg:LMI}

\renewcommand{\algorithmicrequire}{\textbf{Input:}}
\renewcommand{\algorithmicensure}{\textbf{Output:}}

\REQUIRE $\xi$, $n_\text{iter}$, $\sigma$, $\theta$
\ENSURE $y^*$

\STATE neural network prediction $\hat{y}_\theta \leftarrow f_\theta(\xi)$
\STATE initialize $z_0\leftarrow[y_0\quad x_0]^T\in\mathbb{R}^{m+n^2}$

\FOR{$k=0$ \TO $n_{\text{iter}}-1$}
    \STATE $w_{k+1} \leftarrow \Pi_{\mathcal{C}_1}(z_k,\hat{y}_\theta)$ \hfill \eqref{eq:proj-C1}
    \STATE $v_{k+1} \leftarrow \Pi_{\mathcal{C}_2}(2w_{k+1}-z_k)$ \hfill \eqref{eq:proj-C2}
    \STATE $z_{k+1} \leftarrow
    v_{k+1}-w_{k+1}+z_k$
\ENDFOR

\STATE $y^* \leftarrow \Pi_m(\Pi_{\mathcal{C}_1}(z_{n_\text{iter}},\hat{y}_\theta))$
\end{algorithmic}
\end{algorithm}

\subsection{Backpropagation via Implicit Differentiation}
To compute gradients during training, we use an implicit differentiation scheme, introduced in~\cite{grontas2026pinet}, to avoid differentiating through all $n_\text{iter}$ iterations of the Douglas-Rachford algorithm. The gradients of the neural network output $\hat{y}_\theta$ with respect to the parameters $\theta$ can be computed with standard backpropagation approaches, so we focus on the computation of gradients through the Douglas-Rachford operation. That is, we seek an efficient computation for
\begin{equation}\label{eq:diff}\frac{\partial y^*}{\partial\hat{y}_\theta}=\frac{\partial \Pi_m(\Pi_{\mathcal{C}_1}(z_{n_\text{iter}}(\hat{y}_\theta),\hat{y}_\theta))}{\partial\hat{y}_\theta}.\end{equation}
We follow the approach in \cite{grontas2026pinet}, leveraging the implicit function theorem to efficiently compute the vector-Jacobian product (VJP) with the Jacobian in \eqref{eq:diff}. We are specifically interested in calculating the VJP
\begin{equation}
\begin{split}
v^T\frac{\partial \Pi_m(\Pi_{\mathcal{C}_1}(z_{n_\text{iter}}(\hat{y}_\theta),\hat{y}_\theta))}{\partial\hat{y}_\theta}=[v^T\,\,\ \bm{0}]\frac{\partial\Pi_{\mathcal{C}_1}(z_{n_\text{iter}}(\hat{y}_\theta),\hat{y}_\theta)}{\partial\hat{y}_\theta}\\=[v^T\,\,\ \bm{0}]\left[\frac{\partial\Pi_{\mathcal{C}_1}}{\partial z_{n_\text{iter}}}\frac{\partial z_{n_\text{iter}}(\hat{y}_\theta)}{\partial\hat{y}_\theta} +  \frac{\partial\Pi_{\mathcal{C}_1}}{\partial \hat{y}_\theta}\right].
\end{split}
\end{equation}

Since $\Pi_{\mathcal{C}_1}$ is linear in both $z$ and $\hat{y}_\theta$, its gradients with respect to $z_{n_\text{iter}}$ and $\hat{y}_\theta$ are straightforward to compute. We focus on computing the VJP 
$$ v^T\frac{\partial z_{n_\text{iter}}(\hat{y}_\theta)}{\partial\hat{y}_\theta}.$$

Computing $\partial z_k/\partial\hat{y}_\theta$ in general requires differentiation through each iteration of the Douglas-Rachford algorithm, since $z_{k+1}(\hat{y}_\theta)=\Phi(z_{k}(\hat{y}_\theta),\hat{y}_\theta)$, with $\Phi$ being a single Douglas-Rachford iteration. This problem is avoided at the fixed point $z^\star$ where $z^\star(\hat{y}_\theta)=\Phi(z^\star(\hat{y}_\theta),\hat{y}_\theta)$. The implicit function theorem therefore gives 

\begin{equation}\left(I_{m+n^2}-\frac{\partial \Phi}{\partial z}\middle|_{z=z^\star}\right)\frac{\partial z}{\partial{\hat{y}_\theta}}=\frac{\partial \Phi}{\partial{\hat{y}_\theta}}.\end{equation}

Note that since $z^\star$ is computationally intractable, we evaluate $\frac{\partial\Phi}{\partial z}$ at $z_{n_\text{iter}}$ in practice. We could solve this linear system to isolate $\frac{\partial z}{\partial{\hat{y}_\theta}}$, but that adds unnecessary computational burden, since we are more interested in the VJP  $v^T\frac{\partial z}{\partial{\hat{y}_\theta}}$. We instead define a vector $\lambda$ that is the solution to the linear system 

\begin{equation}\lambda^T\left(I_{m+n^2}-\frac{\partial \Phi}{\partial z}\middle|_{z=z^\star}\right)=v^T.\end{equation}
This gives the following VJP of interest.
\begin{equation}v^T\frac{\partial z}{\partial\hat{y}_\theta} = \lambda^T\frac{\partial\Phi}{\partial\hat{y}_\theta}\end{equation}
This strategy of computing gradients through the projection layer via implicit differentiation, rather than considering every iteration in Algorithm~\ref{alg:LMI}, provides an efficient backpropagation scheme for training. 

\section{Convergence Analysis}\label{sec:theory}
In this section, we show that the LMI-Net alternating projection satisfies standard Douglas-Rachford assumptions and therefore converges to a point that satisfies the LMI constraint. 
\begin{lemma}[Eigenvalue clipping as a Euclidean projection]\label{lemma:eigen-clip}

For a symmetric matrix $\bar{X}$ with eigendecomposition $U\Lambda U^T$, the eigenvalue clipping operation $U\text{max}(0,\Lambda)U^T$ is the Euclidean projection onto the positive semidefinite cone. 
\end{lemma}
\begin{proof}
     The Euclidean projection onto the positive semidefinite cone is 
    $$\underset{X}{\text{argmin}}||\bar{X}-X||^2_F \text{ subject to }X\in\mathbb{S}_+$$
    where $\mathbb{S}_+$ denotes the set of all real symmetric positive semidefinite matrices. Note that for a symmetric matrix, the eigenvector matrix $U$ is unitary (i.e., $U^TU=I$). The Frobenius norm is unitarily invariant, which gives
    $$||\bar{X}-X||^2_F=||U^T\bar{X}U-U^TXU||^2_F=||\Lambda-U^TXU||^2_F.$$
    $$\underset{x}{\text{argmin}}||\Lambda-U^TXU||^2_F \text{ subject to }X\in\mathbb{S}_+.$$
    Define $W=U^TXU$. When $X$ is positive semidefinite, $W$ is too. To see this, consider $z^TWz=z^TU^TXUz=\bar{z}^TX\bar{z}$ with $\bar{z}=Uz$. Clearly, when $X\succeq0$, then $W\succeq0$. The Euclidean projection can then be written as $X^*=UW^*U^T$, where
    $$W^*=\underset{W}{\text{argmin}}||\Lambda-W||^2_F \text{ subject to }W\in\mathbb{S}_+$$
    $$=\underset{W}{\text{argmin}}\sum_i\sum_j(\Lambda_{ij}-W_{ij})^2 \text{ subject to }W\in\mathbb{S}_+$$
    $$=\underset{W\in\mathbb{S}_+}{\text{argmin}}\sum_i(\Lambda_{ii}-W_{ii})^2+\sum_{i\neq j}(W_{ij})^2.$$

    The optimal $W$ is therefore diagonal. To see this, consider $D=\text{diag}\left(\left[W_{11},W_{22},\dots,W_{nn}\right]\right)$. When $W\succeq0$, its diagonals are nonnegative, so $D\succeq0$. $D$ never gives a larger objective value than $W$, since the off-diagonals can only increase the objective. Therefore, the optimal $W$ will be diagonal, giving the new objective 
    $$w^*=\underset{w}{\text{argmin}}\sum_i(\lambda_{i}-w_{i})^2 \text{ subject to }w_i\geq0,$$
    where $\lambda$ and $w$ are the diagonal elements of $\Lambda$ and $W$, respectively. Clearly, this objective is minimized with the clipping operator $w=\text{max}(0,\lambda)$. This gives $X^*=UW^*U^T=U\text{max}(0,\Lambda)U^T$, which is equivalent to the eigenvalue clipping operation. 
    
\end{proof}
\begin{remark}[On the symmetry of $\bar{X}$]
    Lemma~\ref{lemma:eigen-clip} requires that $\bar{X}$ is symmetric. That is, the input into the projection onto $\mathcal{C}_2$, defined by \eqref{eq:proj-C2}, must be symmetric. The alternating nature of the Douglas-Rachford algorithm means the projection onto $\mathcal{C}_1$, defined in \eqref{eq:proj-C1}, should output a symmetric $X$. The projection $\Pi_{\mathcal{C}_1}$ is guaranteed to output a symmetric $x$ because it satisfies the constraint $F(y)=X$ by design. $F(y)$ is defined in \eqref{eq:LMI-opt} to be a linear combination of symmetric matrices, so $X$ is symmetric by design. 
\end{remark}
\begin{theorem}
    Assume $\mathcal{C}_1$, $\mathcal{C}_2$ are closed, nonempty, convex sets and $\mathcal{C}_1\cap\mathcal{C}_2\neq\emptyset$. Then the sequence $z_k$ generated by Algorithm~\ref{alg:LMI} converges to a fixed point $z^\star$ of the Douglas-Rachford operator, and the shadow sequence 
    $$w_k=\Pi_{\mathcal{C}_1}(z_k,\hat{y}_\theta)$$
    converges to a point $w^\star\in\mathcal{C}_1\cap\mathcal{C}_2$. In particular, the output $y^\star=\Pi_m(w^\star)$ satisfies the LMI condition $F(y)\succeq0$. 
\end{theorem}
\begin{proof} Since $\mathcal{C}_1$ and $\mathcal{C}_2$ are convex, their indicator functions $\mathcal{I}_{\mathcal{C}_1}(z)$ and $\mathcal{I}_{\mathcal{C}_2}(z)$ are convex, making the objectives $f(z)=||y-\hat{y}_\theta||^2 + \mathcal{I}_{\mathcal{C}_1}(z)$ and $g(z)=\mathcal{I}_{\mathcal{C}_2}(z)$ convex. 

By Lemma~\ref{lemma:eigen-clip}, $\Pi_{\mathcal{C}_2}(z)$ is the true proximal operator for $g(z)=\mathcal{I}_{\mathcal{C}_2}(z)$. By definition, $\Pi_{\mathcal{C}_1}(z,\hat{y}_\theta)$ is the true proximal operator for $f(z)=||y-\hat{y}_\theta||^2 + \mathcal{I}_{\mathcal{C}_1}(z)$. Therefore, Algorithm~\ref{alg:LMI} is exactly an instance of Douglas-Rachford applied to the problem of minimizing $||y-\hat{y}_\theta||^2$ subject to $y\in\mathcal{C}_1\cap\mathcal{C}_2$.
The convergence therefore follows from standard Douglas-Rachford results \cite[Corollary~28.3]{bauschkeconvex}\cite[Corollary~1]{davis2017convergence}.

\end{proof}

\section{Numerical Experiments}

We evaluate LMI-Net on two problems for linear systems under disturbance: (i) invariant ellipsoid synthesis and (ii) joint controller and invariant ellipsoid design. 
For both tasks, we compare LMI-Net against a soft-constrained baseline trained with the same augmented loss described in \eqref{eq:soft-loss} and against CVXPY/SCS \cite{diamond2016cvxpy} as a solver baseline. The comparison metrics we report are constraint violation, runtime, and closed-loop instability when applicable. For ease of exposition, we provide detailed descriptions of the learning problem formulation under LMI constraints, dataset construction, and hyperparameter choice in the appendix.

It is worth noting that at inference time, the fixed LMI-Net can be adapted naturally to different Douglas-Rachford (DR) iterations. The number of DR iterations, therefore, provides a practical tuning parameter that trades off feasibility with computation speed, as the algorithm provably converges to a feasible point under increasing iterations.

\subsection{Invariant Ellipsoid Synthesis}\label{sec:linear_dist}

We first evaluate the disturbed linear-system invariant ellipsoid synthesis problem introduced in Section~\ref{sec:example}. We test on the training distribution and on two out-of-distribution testing datasets: \textsc{OOD-slow}, which moves eigenvalues closer to the imaginary axis and therefore has slower dynamics; \textsc{OOD-large}, which increases the magnitude of the disturbance. Table~\ref{tab:viol_frac} reports violation fractions, and Table~\ref{tab:speed} reports runtime.

The soft-constrained baseline degrades significantly in feasibility under distribution shift, with violation rates of 94.4\% on \textsc{OOD-slow} and 77.7\% on \textsc{OOD-large}. In contrast, LMI-Net improves strict constraint satisfaction monotonically as more DR iterations are used at inference time. At 2000 iterations, violations are already zero on \textsc{Train} and \textsc{OOD-slow}. With 4000 iterations, LMI-Net matches CVXPY feasibility on all three datasets, while remaining 9-35$\times$ faster than CVXPY/SCS. These results show that the hard-constrained approach in LMI-Net substantially improves out-of-distribution feasibility while preserving fast inference.

\begin{table}[ht]
    \centering
    \caption{Constraint violation fraction on training and testing sets}
    \label{tab:viol_frac}
    \begin{tabular}{lccc}
        \toprule
        Method & \textsc{Train}  & \textsc{OOD-slow}  & \textsc{OOD-large}  \\
        \midrule
        Soft constrained model     & 12.0\%   & 94.4\%    & 77.7\% \\
        LMI-Net (DR 500)   & 12.9\%    & 2.8\%     & 26.0\% \\
        LMI-Net (DR 1000)  & 4.9\%     & 1.4\%     & 12.7\% \\
        LMI-Net (DR 2000)  & \textbf{0.0\%}  & \textbf{0.0\%}  & 2.7\% \\
        LMI-Net (DR 3000)  & \textbf{0.0\%}  & \textbf{0.0\%}  & 0.3\% \\
        LMI-Net (DR 4000)  & \textbf{0.0\%}  & \textbf{0.0\%}  & \textbf{0.0}\% \\
        CVXPY/SCS     & 0.0\%       & 0.0\%       & 0.0\% \\
        \bottomrule
    \end{tabular}
\end{table}

\begin{table}[ht]
    \centering
    \caption{Computation time comparison (ms/sample)}
    \label{tab:speed}
    \begin{tabular}{lccc}
        \toprule
        Method & \textsc{Train} & \textsc{OOD-slow} & \textsc{OOD-large} \\
        \midrule
        Soft constrained model     & 0.2  & 0.6  & 0.1 \\
        LMI-Net (DR 500)   & 0.8  & 5.3  & 1.4 \\
        LMI-Net (DR 1000)  & 0.7  & 4.6  & 1.1 \\
        LMI-Net (DR 2000)  & 1.0  & 5.7  & 1.6 \\
        LMI-Net (DR 3000)  & 1.1  & 5.8  & 1.5 \\
        LMI-Net (DR 4000)  & 1.5  & 7.6  & 2.1 \\
        CVXPY/SCS     & 53.3 & 72.0 & 56.8 \\
        \bottomrule
    \end{tabular}
\end{table}

\subsection{Joint Controller and Invariant Ellipsoid Design}

We next consider joint synthesis of a stabilizing feedback controller and an invariant ellipsoid for a disturbed linear system. We test on the training distribution and an out-of-distribution (OOD) testing dataset, which increases the magnitude of unstable eigenvalues in the open-loop dynamics. Tables~\ref{tab:ctrl_train} and~\ref{tab:ctrl_ood} report violation rate, closed-loop instability, and runtime on the training and OOD datasets, respectively.

The soft-constrained baseline fails to satisfy the LMI constraint, and can destabilize the system, especially on OOD samples, where 79.2\% of predictions are infeasible and 56.7\% produce unstable closed-loop dynamics. LMI-Net eliminates closed-loop instability with 1000 DR iterations on both datasets, and continues to improve feasibility as the number of inference-time DR iterations increases. Figure~\ref{fig:traj} further illustrates this contrast on a representative OOD sample: LMI-Net produces a stabilizing controller whose trajectories remain within the certified invariant ellipsoid, while the soft-constrained model outputs a destabilizing gain.

On the training set, LMI-Net reaches zero violations at 3000 iterations while remaining 3.5$\times$ faster than CVXPY/SCS. On the OOD set, its violation percentage drops from 14.6\% at 500 iterations to 3.4\% at 4000 iterations. These observations validate the practical advantage of the LMI-Net, where the number of DR iterations serves as a tunable speed-feasibility tradeoff parameter.

\begin{table}[ht]
    \centering
    \caption{Performance comparison on the training dataset}
    \label{tab:ctrl_train}
    \begin{tabular}{lcccc}
        \toprule
        Method & violation \% & CL unstable \% & ms/sample \\
        \midrule
        Soft constrained model  & 46.6\% & 3.2\%          & 0.003 \\
        LMI-Net (DR 500)      & 3.2\%  & \textbf{0.0\%} & 0.208  \\
        LMI-Net (DR 1000)    & 1.2\%  & \textbf{0.0\%} & 0.414  \\
        LMI-Net (DR 2000)      & 0.6\%  & \textbf{0.0\%}     & 0.826 \\
        LMI-Net (DR 3000)      & \textbf{0.0\%} & \textbf{0.0\%} & 1.234 \\
        LMI-Net (DR 4000)      & \textbf{0.0\%} & \textbf{0.0\%} & 1.638  \\
        CVXPY (SCS)       & \textbf{0.0\%} & \textbf{0.0\%} & 4.290  \\
        \bottomrule
    \end{tabular}
\end{table}

\begin{table}[ht]
    \centering
    \caption{Performance comparison on the out-of-distribution dataset}
    \label{tab:ctrl_ood}
    \begin{tabular}{lcccc}
        \toprule
        Method & violation \% & CL unstable \% & ms/sample \\
        \midrule
        Soft constrained model  & 79.2\% & 56.7\%         & 0.006 \\
        LMI-Net (DR 500)       & 14.6\% & 0.6\%          & 0.331 \\
        LMI-Net (DR 1000)      & 9.0\%  & \textbf{0.0\%} & 0.661 \\
        LMI-Net (DR 2000)      & 5.6\%  & \textbf{0.0\%} & 1.317 \\
        LMI-Net (DR 3000)      & 4.5\%  & \textbf{0.0\%} & 1.973 \\
        LMI-Net (DR 4000)      & 3.4\%  & \textbf{0.0\%} & 2.628 \\
        CVXPY (SCS)       & \textbf{0.0\%} & \textbf{0.0\%} & 5.067 \\
        \bottomrule
    \end{tabular}
\end{table}

\begin{figure}[htbp]
    \centering
    \includegraphics[width=3.5in]{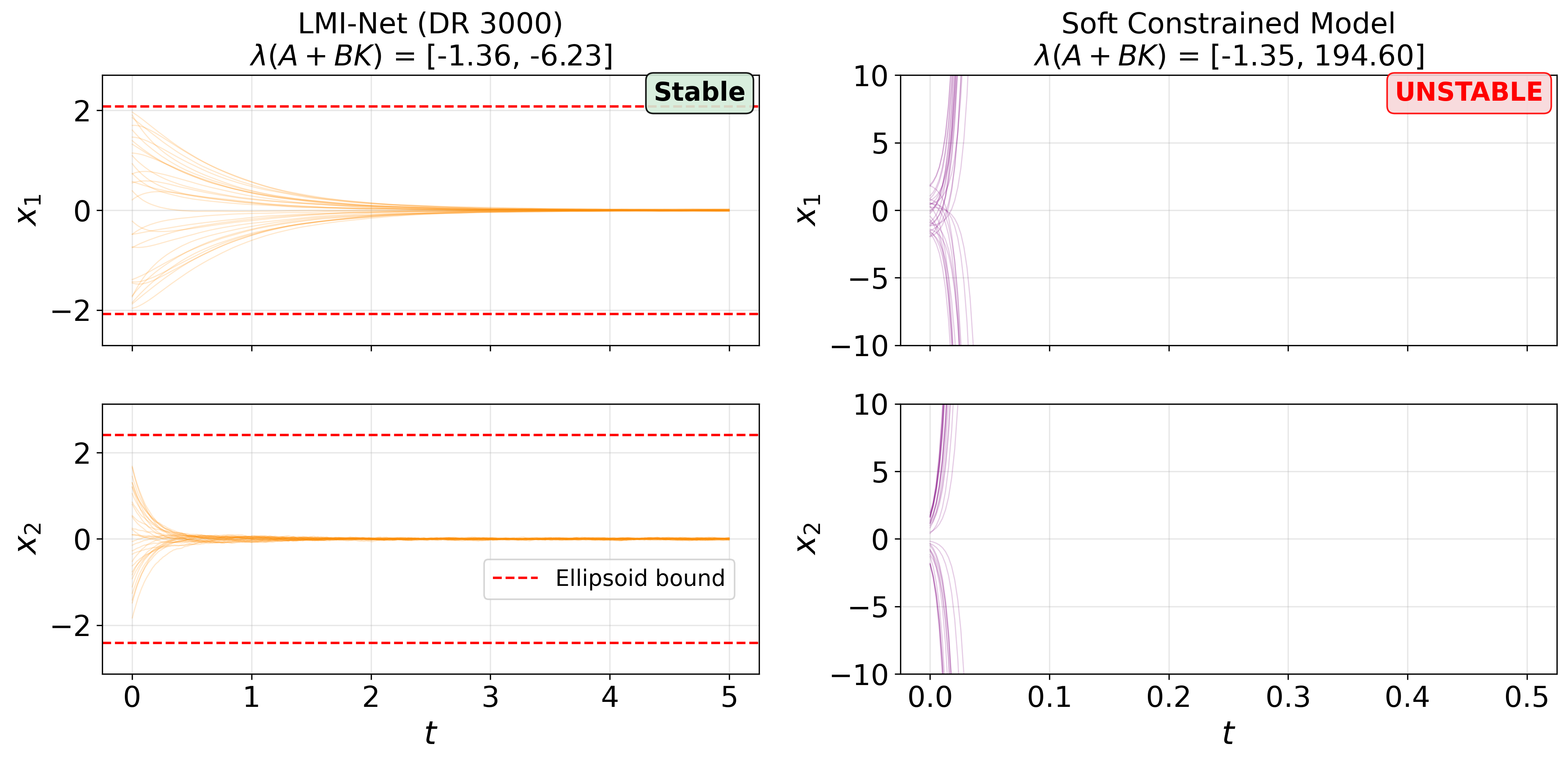}
    \caption{Our proposed LMI-Net with 3000 DR iterations during evaluation jointly learned a stabilizing feedback controller and an invariant ellipsoid, while the soft-constrained model outputs a feedback gain that results in an unstable closed-loop system.}
    \label{fig:traj}
\end{figure}


\section{Conclusions}
We introduced LMI-Net, a modular differentiable projection layer that turns a standard neural network into a feasible-by-construction model that satisfies linear matrix inequality (LMI) constraints.
By decomposing the LMI-constrained set into an affine constraint and a positive semidefinite cone, we leveraged Douglas-Rachford (DR) splitting to design an iterative forward pass and an efficient backward pass through implicit differentiation. We provide theoretical results that establish formal convergence guarantees as the number of DR iterations increases. In the numerical experiments based on classical LMI reformulations, LMI-Net substantially reduced constraint violation instances and improved closed-loop stability compared to soft-constrained models, while retaining lower computation cost over solving each semidefinite program from scratch. The experiments also highlight a practical advantage of the LMI-Net design, where DR iterations provide a simple knob for trading computation for tighter feasibility without retraining. Future work includes scaling to higher-dimensional problems with advanced backbone architectures and extending the framework to practical control tasks such as tube MPC and contraction-metric-based controller synthesis.

\section*{Appendix}

\subsection*{Soft-constrained Approaches for LMI-constrained Learning}
Current soft-constrained approaches incorporate a regularization term that penalizes constraint violation, an example of which is outlined in the following optimization problem.
\begin{align}
    L_{\text{soft}}(\theta) &= \frac{1}{N_{\text{train}}}\sum_{i=1}^{N_{\text{train}}} \bigg(c(\hat{y}(\xi^{(i)};\theta)) \notag\\
    &- \beta_{\text{soft}} \lambda_{\min}(F(\hat{y}(\xi^{(i)}; \theta); \xi^{(i)})) \bigg) \label{eq:soft-loss}
\end{align}
Here, $\lambda_{\min}(\cdot)$ is the minimum eigenvalue of the matrix, and $\beta_{\text{soft}} > 0$ is a weighting parameter. This approach cannot provide guarantees of constraint satisfaction, especially on $\xi$ values outside the training distributions.

\subsection*{Additional Details in Numerical Experiments}
We provide implementation details of numerical experiments in this section. In both experiments, the soft-constrained baseline and LMI-Net use the same two-layer MLP backbone with 64 neurons per layer and ReLU activations, and both are trained with the augmented loss in \eqref{eq:soft-loss} with $\beta_{\text{soft}}=100$. At inference time, the LMI-Net (fixed after training) is run with $\{500,1000,2000,3000,4000\}$ Douglas-Rachford (DR) iterations to study the runtime-feasibility tradeoff without retraining.

\subsection*{Invariant ellipsoid problem}\label{app:inv_details}

We use the linear system under disturbance, introduced in
Section~\ref{sec:example}, with $n_x=2$, $n_w=1$, and fixed $\alpha=0.1$.
When creating the datasets, each matrix $A$ is generated as
\[
A = U \,\mathrm{diag}(\lambda_1,\lambda_2)\, U^\top,
\]
where $\lambda_i \sim \mathrm{Uniform}(-\lambda_{\max},-\lambda_{\min})$
and $U$ is drawn uniformly from the orthogonal group. Each entry of
$B_w$ is sampled independently from $\mathcal{N}(0,\sigma_{B_w}^2)$.

The training distribution uses $(\lambda_{\min},\lambda_{\max},\sigma_{B_w})=(0.5,5.0,1.0)$. For the two out-of-distribution (OOD) testing sets, \textsc{OOD-slow} is generated with $(0.05,0.5,1.0)$, and \textsc{OOD-large} with $(0.5,5.0,3.0)$.

The network maps the flattened input $(A,B_w)$ to the upper-triangular
entries of $P$. The objective $c$ in \eqref{eq:learning-obj} is defined as $-\log\det(P)$, which would minimize the volume of the invariant ellipsoid. Both models are trained with Adam for 500 epochs. For LMI-Net, we use $\sigma=0.1$ and 500 DR iterations during training. A sample is counted towards constraint violation when the maximum eigenvalue of the LMI residual in
the left-hand side of \eqref{eq:exp1-lmi-simple} is positive.

\subsection*{Joint controller and invariant ellipsoid problem}\label{app:ctrl_details}

We consider a linear system with control input under bounded disturbance, assuming that $(A, B)$ is stabilizable:
\begin{equation}
\dot{x} = A x + B u + B_w w, \qquad w^\top w \le 1.
\end{equation}
The goal is to jointly design a feedback gain $K$ and an invariant ellipsoid $\{x: x^\top P x\leq 1\}$ under the control law $u = Kx$. Following the reformulation in \cite{boyd1994linear}, using the change of variables $Q=P^{-1}$ and $Y=KQ$, the problem can be reduced to the following LMI:
\begin{equation}
\begin{bmatrix}
Q A^\top + A Q + Y^\top B^\top + B Y + \alpha Q & B_w \\
B_w^\top & -\alpha I
\end{bmatrix} \preceq 0, Q\succeq \varepsilon I,
\end{equation}
with $\varepsilon = 10^{-3}$, and the fixed S-procedure parameter $\alpha = 0.1$. The two constraints are combined into a single block-diagonal LMI. After solving for $(Q,Y)$, the controller is recovered as $K = Y Q^{-1}$ and the invariant ellipsoid as $\mathcal{E}(Q) = \{x : x^\top Q^{-1} x \leq 1\}$.

Each sample in the training and testing datasets is a tuple $(A,B,B_w)$. Both matrices $B \in \mathbb{R}^{2 \times 1}$ and  $B_w \in \mathbb{R}^{2 \times 1}$ are filled with entries drawn from $\mathcal{N}(0,1)$, same for both training and testing sets. Each eigenvalue magnitude in matrix $A \in \R^{2\times 2}$, $|\lambda_i|$, is drawn uniformly from $[\lambda_{\min}, \lambda_{\max}]$, and its sign is assigned differently in the training set and testing set. The training set draws each eigenvalue magnitude uniformly from $[0.1, 1.0]$, then assigns a positive sign to each eigenvalue with 50\% probability independently. The out-of-distribution (OOD) test set shifts to eigenvalue magnitudes in $[0.3, 1.5]$ with all samples having one unstable eigenvalue. We filter out the samples within these datasets where $(A, B)$ are not stabilizable.

The neural network maps the flattened input $(\operatorname{vech}(A), \operatorname{vec}(B), \operatorname{vec}(B_w)) \in \mathbb{R}^{7}$ to the decision variables $(\operatorname{vech}(Q), \operatorname{vec}(Y)) \in \mathbb{R}^{5}$. The objective $c$ in \eqref{eq:learning-obj} is chosen as $\log\det(Q)$, which minimizes the invariant ellipsoid volume. We train both the soft-constrained model and our LMI-Net with Adam for 1000 epochs on the same training dataset. The LMI-Net is trained with 500 DR iterations and $\sigma = 0.01$. 

The evaluation metric \emph{Violation fraction} refers to the percentage of samples whose maximum eigenvalue violation of the LMI constraint exceeds 0. The metric \emph{CL instability} refers to the percentage of samples for which $A+BK$ has at least one eigenvalue with a positive real part. The metric \emph{Computation time} reports the wall-clock milliseconds per sample, evaluated on a workstation with an Intel Ultra 9 285K CPU and an NVIDIA RTX5080 GPU.





\newpage
\bibliographystyle{ieeetr}
\bibliography{refs}

\end{document}